\documentclass[11pt,reqno,a4paper]{article}

\usepackage{amsmath}
\usepackage{amsfonts}
\usepackage{amssymb}
\usepackage{multirow}
\usepackage[dvips]{graphicx}
\usepackage[dvips]{color}
\usepackage{epsfig}
\usepackage{latexsym}
\usepackage{enumerate}

\usepackage[bookmarksnumbered=true,
citecolor=blue, colorlinks=false, hypertex]{hyperref}

\newtheorem{thm}{Theorem}[section]

\newtheorem{lem}[thm]{Lemma}
\newtheorem{prop}[thm]{Proposition}
\newenvironment{proof}{Proof:}{\quad \hfill $\Box$\vspace{2ex}}

\numberwithin{equation}{section}

\newcommand{\R}{\mathbb{R}}

\def \bI {\Bbb I}
\def \bN {\Bbb N}

\def \bR {\Bbb R}

\def \bx {{\bf x}}
\def \bz {{\bf z}}

\def \cB {{\cal B}}

\def \cD {{\cal D}}

\def \cM {{\cal M}}
\def \cS {{\cal S}}

\def \cN {{\cal N}}

\def \cP {{\cal P}}

\def \cE {{\cal E}}

\def \lx {{L^2_{\rho_X}}}

\def \and {\, \mbox{\rm and}\, }

\def \span {\,{\rm span}\,}
\def \supp {\,{\rm supp}\,}

\def \ran {{\rm ran}\,}

\newcommand{\va}{\boldsymbol{a}}

\newcommand{\vc}{\boldsymbol{c}}

\newcommand{\vs}{\boldsymbol{s}}
\newcommand{\vt}{\boldsymbol{t}}

\newcommand{\vxi}{\boldsymbol{\xi}}
\newcommand{\vy}{\boldsymbol{y}}

\newcommand{\me}{\text{e}}

\title{Reproducing Kernel Banach Spaces with the $\ell^1$ Norm II: Error Analysis for Regularized Least Square Regression\thanks{Supported
by Guangdong Provincial Government of China through the
``Computational Science Innovative Research Team" program.}}

\author{\quad Guohui Song\thanks{School of Mathematical and Statistical Sciences, Arizona State University, Tempe, AZ 85287. E-mail address: {\it gsong9@asu.edu}.}\quad and \quad Haizhang Zhang\thanks{School of Mathematics
and Computational Science and Guangdong Province Key Laboratory of Computational Science,
 Sun Yat-sen University, Guangzhou 510275, P. R. China. E-mail address: {\it zhhaizh2@sysu.edu.cn}.}}

\date{}

\begin{document}
\maketitle

\begin{abstract}
A typical approach in estimating the learning rate of a regularized learning scheme is to bound the approximation error by the sum of the sampling error, the hypothesis error and the regularization error. Using a reproducing kernel space that satisfies the linear representer theorem brings the advantage of discarding the hypothesis error from the sum automatically. Following this direction, we illustrate how reproducing kernel Banach spaces with the $\ell^1$ norm can be applied to improve the learning rate estimate of $\ell^1$-regularization in machine learning.

\vspace*{0.2cm}

\noindent{\bf Keywords}: reproducing kernel Banach spaces, sparse learning, regularization, least square regression, learning rate, the representer theorem
\end{abstract}

\vspace*{0.3cm}

\section{Introduction}
\setcounter{equation}{0}
A class of reproducing kernel Banach spaces (RKBS) with the $\ell^1$ norm that satisfies the linear representer theorem was recently constructed in \cite{SongZhang2011a}. The purpose of this note is to illustrate how the obtained spaces can be applied to estimate the learning rate of the $\ell^1$-regularized least square regression in machine learning.

A general coefficient-based regularization of the least square regression has the form
\begin{equation}\label{regularization1}
\min_{\vc\in\bR^m}\frac1m\sum_{j=1}^m|K^\bx(x_j)\vc-y_j|^2+\lambda \phi(\vc),
\end{equation}
where $\bx:=\{x_j:j\in\bN_m\}$ with $\bN_m:=\{1,2,\ldots,m\}$ is the sequence of sampling points from an input space $X$, $y_j\in Y\subseteq\bR$ is the observed data on $x_j$, $\lambda$ is a positive regularization parameter, $\phi$ is a nonnegative regularization function on the coefficient column vector $\vc$, and with a chosen function $K:X\times X\to\bR$, $K^\bx(x)$ is the $1\times m$ row vector $(K(x_j,x):j\in\bN_m)$.

When $K$ is a positive-definite reproducing kernel on $X$ and
\begin{equation}\label{rkhsregularizer}
\phi(\vc):=\vc^TK[\bx]\vc,
\end{equation}
where $K[\bx]$ is the $m\times m$ matrix defined by
$$
(K[\bx])_{j,k}:=K(x_k,x_j),\ \ j,k\in\bN_m,
$$
it follows from the celebrated representer theorem \cite{Kimeldorf1971} that (\ref{regularization1}) is the classical regularization network and has been extensively studied in the literature \cite{Evgeniou2000,Scholkopf2001a,Shawe-Taylor2004,Song2010,Vapnik1998}. Estimates for the learning rate of the regularization network can be found, for example, in \cite{CuckerSmale2002,CuckerZhou2007,SmaleZhou2007,SunWu2010,ZhangT}. Learning rates for (\ref{regularization1}) when $\phi(\vc)=\sum_{j=1}^m |c_j|^p$ for $1<p\le 2$ and $p=2$ were respectively obtained in \cite{TongChen} and \cite{SunWu2}. The linear programming regularization where $\phi(\vc)$ is the $\ell^1$ norm $\|\vc\|_1$ of $\vc$ has recently attracted much attention. The increasing interest is mainly brought by the progress of the lasso in statistics \cite{Tibshirani1996} and compressive sensing \cite{Cand`es2006,Chen1998} in which $\ell^1$-regularization is able to yield sparse representation of the resulting minimizer, a desirable feature in model selection. Moreover, the $\ell^1$-regularization is particularly robust to non-Gaussian additive noise such as impulsive noise \cite{Alliney1997,Nikolova2004}.

Without making use of a reproducing kernel space, the recent references \cite{Shi2011,Xiao2010} established estimates of the learning rate for the $\ell^1$-regularized least square regression
\begin{equation}\label{regularization2}
\min_{\vc\in\bR^m}\frac1m\sum_{j=1}^m|K^\bx(x_j)\vc-y_j|^2+\lambda \|\vc\|_1.
\end{equation}
We attempt to show that improvement on the estimates could be made if an RKBS with the $\ell^1$ norm is used. To explain how this could be done, we first introduce the popular approach \cite{CuckerZhou2007} for learning rate estimates in machine learning.

A fundamental assumption in machine learning is that the sample data $\bz:=\{(x_j,y_j):j\in\bN_m\}\in X\times Y$ consists of independent and identically distributed instances of a random variable $(x,y)\in X\times Y$ subject to an unknown probability measure $\rho$ on $X\times Y$. The performance of a predictor $f:X\to Y$ is hence measured by
$$
\cE(f):=\int_{X\times Y}|f(x)-y|^2d\rho.
$$
The predictor that minimizes the above error is the regression function
\begin{equation}\label{frho}
f_\rho(x):=\int_Y yd\rho(y|x),\ \ x\in X,
\end{equation}
where $\rho(y|x)$ denotes the conditional probability measure of $y$ with respect to $x$. In fact, we have for every predictor $f$ that
\begin{equation}\label{orthogonal}
\cE(f)=\cE(f_\rho)+\|f-f_\rho\|_{L^2_{\rho_X}}^2,
\end{equation}
where $\rho_X$ is the marginal probability measure of $\rho$ on $X$ and for $p\in[1,+\infty)$, $L^p_{\rho_X}$ denotes the Banach space of measurable functions $f$ on $X$ with respect to $\rho_X$ such that
$$
\|f\|_{L^p_{\rho_X}}:=\biggl(\int_{X}|f(x)|^pd\rho_X(x)\biggr)^{1/p}<+\infty.
$$
The formula \eqref{frho}, though attractive, is only of theoretical value as $\rho$ is unknown. A practical way is to find a minimizer $\vc_{\bz,\lambda}$ of (\ref{regularization1}) and hope that
\begin{equation}\label{fzlambda}
f_{\bz,\lambda}(x):=K^\bx(x)\vc_{\bz,\lambda},\ \ x\in X
\end{equation}
will be competitive with $f_\rho$ in the sense that the approximation error $\cE(f_{\bz,\lambda})-\cE(f_\rho)$ would be small. To be more precise, for the learning scheme (\ref{regularization1}) to be useful in practice, this error should converge to zero fast in probability as the number of sampling points increases.

The approach in \cite{CuckerZhou2007} works by introducing intermediate functions between $f_{\bz,\lambda}$ and $f_\rho$ that are from a Banach space $\cB$ of functions on $X$ with the properties that $K(x,\cdot)\in\cB$ for all $x\in X$ and for all pairwise distinct $x_j\in X$, $j\in\bN^m$ and $\vc\in\bR^m$
$$
\psi(\|K^\bx(\cdot)\vc\|_\cB)=\phi(\vc),
$$
for some nonnegative function $\psi$. Here $\|\cdot\|_\cB$ is the norm on $\cB$. Let $g$ be an arbitrary function from such a space $\cB$ and set for each function $f:X\to \bR$
$$
\cE_\bz(f):=\frac1m\sum_{j=1}^m(f(x_j)-y_j)^2.
$$
The approximation error $\cE(f_{\bz,\lambda})-\cE(f_\rho)$ can then be decomposed into the sum of four quantities
\begin{equation}\label{boundbythreeerrors}
\cE(f_{\bz,\lambda})-\cE(f_\rho)=\cS(\bz,\lambda,g)+\cP(\bz,\lambda,g)+\cD(\lambda,g)-\lambda\psi(\|f_{\bz,\lambda}\|_\cB),
\end{equation}
where
$$
\begin{array}{ll}
\cS(\bz,\lambda,g)&:=\cE(f_{\bz,\lambda})-\cE_\bz(f_{\bz,\lambda})+\cE_\bz(g)-\cE(g),\\
\cP(\bz,\lambda,g)&:=\left(\cE_{\bz}(f_{\bz,\lambda})+\lambda \psi(\|f_{\bz,\lambda}\|_\cB)\right)-\left(\cE_\bz(g)+\lambda \psi(\|g\|_\cB)\right),\\
\cD(\lambda,g)&:=\cE(g)-\cE(f_\rho)+\lambda \psi(\|g\|_\cB).
\end{array}
$$
The above three quantities are called the {\it sampling error}, the {\it hypothesis error} and the {\it regularization error}, respectively. The strategy is to choose $\cB$ and $g$ carefully so that these three errors can be well bounded from above. When $\cB$ is the reproducing kernel Hilbert space of a positive-definite reproducing kernel $K$ on $X$ and the regularizer $\phi$ is given by (\ref{rkhsregularizer}), we have $\psi(t)=t^2$, $t\in\bR$ and by the representer theorem and the definition of $f_{\bz,\lambda}$ in \eqref{fzlambda} that
\begin{equation}\label{representer}
f_{\bz,\lambda}=\arg \min_{f\in\cB}\left(\cE_\bz(f)+\lambda \|f\|_\cB^2\right).
\end{equation}
In this case, one immediately has that $\cP(\bz,\lambda,g)\le 0$ and thus, by (\ref{boundbythreeerrors}) that
\begin{equation}\label{justtwoterms}
\cE(f_{\bz,\lambda})\le \cS(\bz,\lambda,g)+\cD(\lambda,g).
\end{equation}
For the $\ell^1$-regularization where $\phi(\vc)=\|\vc\|_1$, the space $\cB$ chosen in \cite{Shi2011,Xiao2010} does not satisfy the linear representer theorem. Consequently, the hypothesis error needed to be dealt with there.

A class of RKBS with the $\ell^1$ norm that satisfies the linear representer theorem was recently constructed in \cite{SongZhang2011a}. In Section 2, we shall follow a similar idea to construct a slightly larger RKBS with the same desirable properties. By using the constructed space, we enjoy the same advantage as that for the RKHS case of discarding the hypothesis error automatically. Moreover, the space also leads to a better estimate of the regularization error than that in \cite{Xiao2010}. Combining these two improvements and directly using the estimates of the sampling error established in \cite{Xiao2010} or \cite{Shi2011}, one immediately has a superior learning rate. As our focus is on the advantages brought by the constructed RKBS, we shall only improve the learning rate estimate of \cite{Xiao2010} in Section 3. Interested readers may follow our strategy to engage the more sophisticated sampling error estimate given in \cite{Shi2011} to improve the learning rate therein.

\section{RKBS by Borel Measures}
\setcounter{equation}{0}

In this section, we construct RKBS applicable to the error analysis of the $\ell^1$-regularized least square regression. The constructed spaces are expected to have the $\ell^1$ norm and satisfy the linear representer theorem. The approach is different from the one by semi-inner products in \cite{Zhang2009,Zhangacha} as an infinite-dimensional $\ell^1$ space is neither reflexive nor strictly convex.

Suppose that the input space $X$ is a locally convex topological space and denote by $C_0(X)$ the space of continuous functions $f:X\to\bR$ such that for all $\varepsilon>0$, the set $\{x\in X:|f(x)|>\varepsilon\}$ is compact. We also impose the requirement that for all pairwise distinct $x_j\in X$, $j\in\bN_m$, $m\in\bN$, the kernel matrix $K[\bx]$ is nonsingular. With the maximum norm $\|f\|_{{C_0(X)}}:=\max_{x\in X}|f(x)|$, the space $C_0(X)$ is a Banach space. Its dual space is isometrically isomorphic to the space $\cM(X)$ of all the signed Borel measures on $X$ with bounded total variation. In other words, for each continuous linear functional $T$ on $C_0(X)$, there exists a unique measure $\mu\in\cM(X)$ such that
\begin{equation}\label{functionalonc0}
T(f)=\int_X f(x)d\mu(x)\mbox{ and }\sup_{f\in C_0(X),f\ne0}\frac{|Tf|}{\|f\|_{{C_0(X)}}}=\|\mu\|,
\end{equation}
where $\|\mu\|$ denotes the total variation of $\mu$.

Let $K$ be a real-valued function on $X\times X$ such that $K(\cdot,x)\in C_0(X)$ for all $x\in X$ and
\begin{equation}\label{denseness}
\overline{\span}\{K(\cdot,x):x\in X\}=C_0(X).
\end{equation}
With such a function, we introduce the following space
\begin{equation}\label{spacecb}
\cB:=\left\{f_\mu:=\int_X K(t,\cdot)d\mu(t):\ \mu\in\cM(X)\right\}
\end{equation}
with the norm
\begin{equation}\label{normspacecb}
\|f_\mu\|_\cB:=\|\mu\|.
\end{equation}
Recall that a vector space $V$ is called a {\it pre-RKBS} \cite{SongZhang2011a} on $X$ if it is a Banach space consisting of functions on $X$ such that point evaluation functionals are continuous on $V$ and such that for all $f\in V$, $\|f\|_V=0$ if and only if $f$ vanishes everywhere on $X$.

\begin{prop}\label{prerkbs}
Suppose that $K(\cdot,x)\in C_0(X)$ for all $x\in X$ and (\ref{denseness}) is satisfied. Then $\cB$ defined by (\ref{spacecb}) is a pre-RKBS on $X$.
\end{prop}
\begin{proof}
We first show that the norm (\ref{normspacecb}) is well-defined. Let $\mu,\nu$ be two measures in $\cM(X)$ such that $f_\mu(x)=f_\nu(x)$ for all $x\in X$. Then we get that
$$
\int_X K(t,x)d(\mu-\nu)(t)=0\mbox{ for all }x\in X.
$$
By the denseness condition (\ref{denseness}), the above equation implies that $\mu-\nu=0$. Thus, the measure $\mu$ associated with a function $f_\mu\in \cB$ is unique. This proves that \eqref{normspacecb} is well-defined and that $\|f_\mu\|_\cB=0$ if and only if $f_\mu(x)=0$ for all $x\in X$. Another consequence is that $\cB$ is isometrically isomorphic to $\cM(X)$ and is hence a Banach space. Finally, we observe for all $x_0\in X$ and $\mu\in \cM(X)$ that
$$
|f_\mu(x_0)|=\left| \int_X K(t,x_0)d\mu(t)\right|\le \|K(\cdot,x_0)\|_{C_0(X)} \|\mu\|=\|K(\cdot,x_0)\|_{C_0(X)} \|f_\mu\|_\cB.
$$
Therefore, point evaluations are continuous linear functionals on $\cB$. We conclude that $\cB$ is a pre-RKBS on $X$. The proof is complete.
\end{proof}

Let the sampling points in $\bx$ be pairwise distinct. By definition, $K^\bx(\cdot)\vc\in\cB$ for all $\vc\in \bR^m$. The denseness condition (\ref{denseness}) implies that $K(x_j,\cdot)$, $j\in\bN_m$ are linearly independent. As a result,
\begin{equation}\label{finitenorm}
\|K^\bx(\cdot)\vc\|_\cB=\|\vc\|_1.
\end{equation}
It is in the above sense that $\cB$ is said to possess the $\ell^1$ norm.

We next turn to the crucial linear representer theorem in $\cB$. We say that $\cB$ satisfies the linear representer theorem if for all continuous nonnegative loss function $Q$ and regularizer $\psi$ with $\lim_{t\to\infty}\psi(t)=+\infty$, the regularized learning scheme
$$
\inf_{f\in\cB}Q(f(\bx))+\lambda \psi(\|f\|_\cB)
$$
has a minimizer $f_0$ of the form $f_0=K^\bx(\cdot)\vc$ for some $\vc\in\bR^m$. Here, $f(\bx)=(f(x_j):j\in\bN_m)^T$.

The following lemma can be proved by arguments similar to those in \cite{SongZhang2011a}.

\begin{lem}\label{mni}
The space $\cB$ satisfies the linear representer theorem if and only if for all $\bx$ of pairwise distinct sampling points and $\vy\in\bR^m$, the minimal norm interpolation
\begin{equation}\label{mniproblem}
\inf\{\|f\|_\cB:f\in\cB,\ f(\bx)=\vy\}
\end{equation}
has a minimizer $f_0$ of the form $f_0=K^\bx(\cdot)\vc$ for some $\vc\in\bR^m$.
\end{lem}

A subspace of $\cB$ was constructed in \cite{SongZhang2011a} and conditions for it to satisfy the linear representer theorem were studied. In order to make use of the results obtained there, we first introduce the subspace. Denote by $\ell^1(X)$ the subset of $\cM(X)$ of those Borel measures that are supported on a countable subset of $X$. Thus, for each $\nu\in\ell^1(X)$, there exist some pairwise distinct points $x_j\in X$, $j\in\bI$ where $\bI$ is a countable index set, such that
$$
\nu(A)=\sum_{x_j\in A}\nu(x_j)\mbox{ for every Borel subset }A\subseteq X.
$$
Denote by $\supp\nu$ the countable set of points where $\nu$ is nonzero. The space $\cB_1$ considered in \cite{SongZhang2011a} is
$$
\cB_1:=\left\{\sum_{x\in \supp\nu}\nu(x)K(x,\cdot):\nu\in\ell^1(X)\right\}
$$
with the norm inherited from that of $\cB$.

Put for all $x\in X$, $K_\bx(x):=(K(x,x_j):j\in\bN_m)^T$, which is an $m\times 1$ vector in $\bR^m$. One should not confuse $K_\bx(x)$ with $K^\bx(x)$. The latter is $1\times m$ and might even not be the transpose of the former as $K$ is not required to be symmetric. The following result about $\cB_1$ is from \cite{SongZhang2011a}.

\begin{lem}\label{representercb1}
For all $\vy\in\bR^m$, the minimal norm interpolation
\begin{equation}\label{mni2}
\inf\{\|f\|_{\cB_1}:f\in\cB_1,\ f(\bx)=\vy\}
\end{equation}
has a minimizer $f_0$ of the form $f_0=K^\bx(\cdot)\vc$ for some $\vc\in\bR^m$ if and only if
\begin{equation}\label{H2condition}
\|K[\bx]^{-1}K_\bx(x)\|_1\le 1\mbox{ for all }x\in X.
\end{equation}
Moreover, under condition (\ref{H2condition}), there holds for all $\vc\in\bR^m$ that
\begin{equation}\label{anorm}
\|\vc^T K_\bx(\cdot)\|_{C_0(X)}=\|\vc^TK[\bx]\|_{\infty},
\end{equation}
where $\|\cdot\|_\infty$ is the maximum norm on $\bR^m$.
\end{lem}

We are ready to present the main result of this section.

\begin{thm}\label{representercb}
The space $\cB$ satisfies the linear representer theorem if and only if (\ref{H2condition}) holds true.
\end{thm}
\begin{proof}
Suppose that (\ref{H2condition}) holds true. By Lemma \ref{mni}, to show that $\cB$ satisfies the linear representer theorem, it suffices to show that $f_0=K^\bx(\cdot)K[\bx]^{-1}\vy$ is a minimizer of (\ref{mniproblem}). Clearly, $f_0(\bx)=\vy$. Let $f_\mu$, $\mu\in\cM(X)$, be an arbitrary function in $\cB$ that satisfies the interpolation condition $f_\mu(\bx)=\vy$. We then have for all $\vc\in\bR^m$ that
$$
\int_X \vc^T K_\bx(t)d\mu(t)=\int_X \sum_{j=1}^m c_jK(t,x_j)d\mu(t)=\sum_{j=1}^m c_j f_\mu(x_j)=\vc^T \vy.
$$
It follows from (\ref{functionalonc0}) that for all $\vc\in \bR^m$
$$
|\vc^T\vy|\le \|\vc^T K_\bx(\cdot)\|_{C_0(X)} \| \mu\|.
$$
This together with (\ref{anorm}) implies that
$$
\|\mu\|\ge \sup_{\vc\in\bR^m,\vc\ne0}\frac{|\vc^T \vy|}{\|\vc^T K_\bx(\cdot)\|_{C_0(X)}}=\sup_{\vc\in\bR^m,\vc\ne0}\frac{|\vc^T \vy|}{\|\vc^TK[\bx]\|_{\infty}}=\sup_{\va\in\bR^m,\va\ne0}\frac{|\va^TK[\bx]^{-1} \vy|}{\|\va\|_{\infty}}=\|K[\bx]^{-1}\vy\|_1.
$$
Now, recall by (\ref{finitenorm}) that $\|f_0\|_\cB=\|K[\bx]^{-1}\vy\|_1$ and by definition of $\| \cdot\|_\cB$ that $\|f_\mu\|_\cB=\|\mu\|$. These two facts combined with the above inequality imply that $\|f_\mu\|_\cB\ge \|f_0\|_\cB$. Thus, $f_0$ is indeed a minimizer of (\ref{mniproblem}).

On the other hand, suppose that $\cB$ satisfies the linear representer theorem and we want to prove (\ref{H2condition}). Let $\vy\in\bR^m$. By Lemma \ref{mni}, the minimal norm interpolation (\ref{mniproblem}) has a minimizer $f_0$ of the form $f_0=K^\bx(\cdot)\vc$ for some $\vc\in\bR^m$. Clearly, $f_0$ is also a minimizer of (\ref{mni2}) because $f_0\in\cB_1$ and
$$
\|f_0\|_{\cB_1}=\|f_0\|_\cB=\inf\{\|f\|_\cB:\ f\in\cB,\ f(\bx)=\vy\}\le \inf\{\|f\|_{\cB_1}:f\in\cB_1,\ f(\bx)=\vy\}.
$$
By Lemma \ref{representercb1}, (\ref{H2condition}) holds true. The proof is complete.
\end{proof}

It will become clear in the next section that the above theorem makes $\cB$ a useful space for error analysis of the $\ell^1$-regularized least square regression.

We present two examples of $K$ that satisfy all the assumptions, especially (\ref{H2condition}), in this section:
\begin{itemize}
\item[--]the exponential kernel
$$
K(s,t):=\me^{-|s-t|}, \quad s,t\in \R,
$$
\item[--]the Brownian bridge kernel
\begin{equation*}
K(s,t):= \min\{s,t \} -st, \quad s,t\in (0,1).
\end{equation*}
\end{itemize}
That these two kernels satisfy (\ref{H2condition}) has been proved in \cite{SongZhang2011a}. It remains to verify the denseness requirement (\ref{denseness}). The exponential kernel is a particular case of the following result.

\begin{prop}\label{translationinvariant}
If $\phi$ is Lebesgue integrable on $\bR^d$ that is nonzero almost everywhere then the function
\begin{equation}\label{fouriertransform}
K(\vs,\vt):=\int_{\bR^d}e^{-i(\vs-\vt)\cdot\vxi}\phi(\vxi)d\vxi,\ \ \vs,\vt\in\bR^d
\end{equation}
satisfies that $K(\cdot,\vt)\in C_0(\bR^d)$ for all $\vt\in \bR^d$ and the denseness condition (\ref{denseness}). So does $K(\vs,\vt):=\psi(\vs-\vt)$, $\vs,\vt\in\bR^d$ where $\psi$ is a nontrivial continuous function on $\bR^d$ of compact support.
\end{prop}
\begin{proof}
That the function given by (\ref{fouriertransform}) belongs to $C_0(\bR^d)$ for all $\vt\in\bR^d$ follows from the Riemann-Lebesgue lemma. The denseness condition (\ref{denseness}) for the two kernels can be proved by arguments similar to those in \cite{SongZhang2011a}.
\end{proof}

The Brownian bridge kernel is handled with a manner different from that in \cite{SongZhang2011a}.

\begin{prop}\label{brownian}
The Brownian bridge kernel satisfies (\ref{denseness}).
\end{prop}
\begin{proof}
Clearly, for the Brownian bridge kernel, $K(\cdot,t)$ is continuous for all $t\in (0,1)$. Let $\nu$ be a Borel measure on $X:=(0,1)$ such that
\begin{equation}\label{brownianeq1}
\int_X K(s,t)d\nu(s)=0\ \mbox{ for all }t\in(0,1).
\end{equation}
Note that $K$ has the representation
$$
K(s,t)=\int_X \Gamma_s(z)\Gamma_t(z)dz,\ \ s,t\in(0,1),
$$
where $\Gamma_s:=\chi_{(0,s)}-s$ with $\chi_{(0,s)}$ denoting the characteristic function of $(0,s)$. Arguments similar to those in \cite{SongZhang2011a} yield that there exists a constant $C$ such that
$$
\int_0^s d\nu(s)=C\mbox{ for all }s\in(0,1).
$$
It follows that $\nu((s_1,s_2))=0$ for all $0<s_1<s_2<1$. Consequently, $\nu$ is the zero Borel measure on $(0,1)$. Thus, the Brownian bridge kernel satisfies (\ref{denseness}).
\end{proof}

Finally, we remark that the function $K$ can be regarded as the reproducing kernel for $\cB$ constructed by (\ref{spacecb}). To see this, we introduce a bilinear form on $\cB\times C_0(X)$ by setting
$$
\langle f_\mu,g\rangle:=\int_X g(x)d\mu(x)\ \mbox{ for all }\mu\in\cM(X)\mbox{ and }g\in C_0(X).
$$
We observe by (\ref{functionalonc0}) that
$$
|\langle f_\mu,g\rangle|\le \|\mu\|\|g\|_{C_0(X)}=\|f_\mu\|_{\cB}\|g\|_{C_0(X)}
$$
and that for all $x\in X$,
$$
f_\mu(x)=\langle f_\mu,K(\cdot,x)\rangle,\ \ g(x)=\langle K(x,\cdot),g\rangle.
$$
In the above senses, $K$ is said to be the reproducing kernel for both $\cB$ and $C_0(X)$.

\section{Error Analysis of the $\ell^1$-Regularization}
\setcounter{equation}{0}

We apply the constructed space $\cB$ to estimate the learning rate of the $\ell^1$-regularized least square regression (\ref{regularization2}) in this section. To this end, we first introduce some standard assumptions in the literature imposed on the regression function $f_\rho$, the input space $X$ and the function $K$.

Let $X$ be compact metric space with the distance $d$ and assume that $\rho_X$ is a Borel probability measure on $X$. In this note, we suppose that $K$ is a positive-definite reproducing kernel on $X$ with the Lipschitz condition
\begin{equation}\label{lip}
|K(x,t)-K(x,t')|\le C_\alpha (d(t,t'))^\alpha\mbox{ for some positive constants }\alpha,C_\alpha\mbox{ and for all }x,t,t'\in X.
\end{equation}
Denote for all $r>0$ by $\cN(X,r)$ the least number of open balls with radius $r$ that cover $X$. Assume that this covering number satisfies for some positive constants $\eta,C_\eta$ that
\begin{equation}\label{covering}
\cN(X,r)\le \frac{C_\eta}{r^\eta}\ \mbox{ for all }0<r\le 1.
\end{equation}
The requirement on $f_\rho$ is that it is contained in the range $\ran(L_K^s)$ of $L_K^s$ for some $s>0$. Here, $L_K$ is the compact positive operator on $L^2_{\rho_X}$ defined by
$$
L_Kf:=\int_{X}K(t,\cdot)f(t)d\rho_X(t),\ \ f\in L^2_{\rho_X}.
$$
Let $\phi_j$, $j\in\bN$ be an orthonormal basis for $L^2_{\rho_X}$ consisting of eigenfunctions of $L_K$ with the corresponding eigenvalues $\lambda_j\ge \lambda_{j+1}$, $j\in\bN$. The assumption $f_\rho\in\ran(L_K^s)$ implies that
$$
f_\rho=\sum_{j=1}^\infty \lambda^s_j a_j\phi_j
$$
for some $h=\sum_{j=1}^\infty a_j\phi_j$ in $L^2_{\rho_X}$. In order to make use of the space constructed in the last section, our last requirement is that $K$ satisfies that $\overline{\span}\{K(\cdot,x):x\in X\}=C(X)$ and condition (\ref{H2condition}).

Let $\vc_{\bz,\lambda}$ be a minimizer of (\ref{regularization2}) and let $f_{\bz,\lambda}$ be given by \eqref{fzlambda}. For the minimization problem (\ref{regularization2}), the hypothesis error and regularization error have the specific forms
$$
\begin{array}{ll}
\cP(\bz,\lambda,g)&:=\left(\cE_{\bz}(f_{\bz,\lambda})+\lambda \|f_{\bz,\lambda}\|_\cB\right)-\left(\cE_\bz(g)+\lambda \|g\|_\cB\right),\\
\cD(\lambda,g)&:=\cE(g)-\cE(f_\rho)+\lambda \|g\|_\cB,
\end{array}
$$
where $g$ is a function in $\cB$ to be carefully chosen.

The use of the space $\cB$ enables us to discard the hypothesis error immediately.
\begin{lem}\label{twoterms}
Under the above assumptions on $K$, there holds $\cE(f_{\bz,\lambda})-\cE(f_\rho)\le \cS(\bz,\lambda,g)+\cD(\lambda,g)$ for all $g\in \cB$.
\end{lem}
\begin{proof}
By Theorem \ref{representercb},
$$
f_{\bz,\lambda}=\arg\min_{f\in\cB}\cE_\bz(f)+\lambda \|f\|_\cB.
$$
As a consequence, $\cP(\bz,\lambda,g)\le0$, which together with inequality (\ref{boundbythreeerrors}) completes the proof.
\end{proof}

We next estimate the regularization error.

\begin{lem}\label{regularizationerror}
If $0<s< 1$ then
\begin{equation}\label{dlambda1s}
\inf_{g\in\cB}\cD(\lambda,g)\le (\|h\|_{L^2_{\rho_X}}+\|h\|_{L^2_{\rho_X}}^2)\lambda^{\frac{2s}{1+s}}.
\end{equation}
If $s\ge1$ then $f_\rho\in\cB$ and
\begin{equation}\label{dlambdas1}
\cD(\lambda,f_\rho)\le (\lambda_1^{s-1}\|h\|_{L^2_{\rho_X}})\lambda.
\end{equation}
\end{lem}
\begin{proof}
Firstly, we have for each $\varphi\in L^2_{\rho_X}$ that $L_K\varphi\in \cB$ and by the Cauchy-Schwartz inequality that
\begin{equation}\label{regularizationerroreq1}
\|L_K\varphi\|_\cB= \|\varphi\|_{L^1_{\rho_X}}\le \|\varphi\|_{L^2_{\rho_X}}.
\end{equation}
If $s\ge1$ then $f_\rho=L_K\varphi$ where
$$
\varphi=\sum_{j=1}^\infty \lambda_j^{s-1}a_j\phi_j.
$$
As $\lambda_j$ is non-increasing,
$$
\|\varphi\|_{L^2_{\rho_X}}\le \lambda_1^{s-1}\biggl(\sum_{j=1}^\infty|a_j|^2\biggr)^{1/2}=\lambda_1^{s-1}\|h\|_{\lx}.
$$
We then get by the above equation and (\ref{regularizationerroreq1}) that
$$
\cD(\lambda,f_\rho)=\lambda\|f_\rho\|_\cB\le\lambda\|\varphi\|_{\lx}\le\lambda \lambda_1^{s-1}\|h\|_{\lx},
$$
which is (\ref{dlambdas1}).

Suppose now that $0<s<1$. If $\lambda_1\le\lambda^{\frac1{1+s}}$ then by (\ref{orthogonal}),
$$
\cD(\lambda,0)=\cE(0)-\cE(f_\rho)=\|f_\rho\|_{\lx}^2=\sum_{j=1}^\infty \lambda_j^{2s}a_j^2\le \lambda^{\frac{2s}{1+s}}\|h\|_{\lx}^2,
$$
which implies (\ref{dlambda1s}). If $\lambda_1>\lambda^{\frac1{1+s}}$ then since $\lambda_j$ decreases to zero as $j$ tends to infinity, there exists some $N\in\bN$ such that $\lambda_{N+1}<\lambda^{\frac1{1+s}}\le \lambda_N$. Put
$$
\varphi:=\sum_{j=1}^N\lambda_j^{s-1}a_j\phi_j.
$$
It follows from (\ref{orthogonal}) and (\ref{regularizationerroreq1}) that
$$
\cD(\lambda,L_K\varphi)\le \|L_K\varphi-f_\rho\|_{\lx}^2+\lambda \|\varphi\|_{L^2_{\rho_X}}.
$$
We estimate that
$$
\lambda\|\varphi\|_{L^2_{\rho_X}}=\lambda\biggl(\sum_{j=1}^N a_j^2\lambda_j^{2s-2}\biggr)^{1/2}\le \lambda \lambda^{\frac{s-1}{1+s}}
\biggl(\sum_{j=1}^N a_j^2\biggr)^{1/2}\le \lambda^{\frac{2s}{1+s}}\|h\|_{\lx}
$$
and that
$$
\|L_K\varphi-f_\rho\|_{\lx}^2=\sum_{j=N+1}^\infty \lambda_j^{2s}a_j^2\le \lambda^{\frac{2s}{1+s}}\sum_{j=N+1}^\infty a_j^2\le \lambda^{\frac{2s}{1+s}}\|h\|_{\lx}^2.
$$
Combing the above two inequalities leads to (\ref{dlambda1s}). The proof is complete.
\end{proof}

We remark that the estimated regularization error in \cite{Xiao2010} was of the order $O(\lambda^{\frac{2s}{2+s}})$ for $0<s\le 2$.

Turning to the sampling error, we follow the approach in \cite{Xiao2010} to decompose it into the sum $\cS(\bz,\lambda,g)=\cS_1(\bz,\lambda,g)+\cS_2(\bz,\lambda)$ where
$$
\cS_1(\bz,\lambda,g)=(\cE_\bz(g)-\cE_\bz(f_\rho))-(\cE(g)-\cE(f_\rho)),\ \ \cS_2(\bz,\lambda)=(\cE(f_{\bz,\lambda})-\cE(f_\rho))-(\cE_\bz(f_{\bz,\lambda})-\cE_\bz(f_\rho)).
$$
The first summand $\cS_1(\bz,\lambda,g)$ can be bounded by using the law of large numbers. By the same arguments as those in \cite{CuckerZhou2007,Xiao2010}, we use the estimate in Lemma \ref{regularizationerror} to obtain an improved bound.

\begin{lem}\label{samplingerror1}
Suppose that the output of sample data is bounded by a positive constant almost surely. If $0<s<1$ then for each $\varepsilon>0$ there exists some $g\in\cB$ such that for all $0<\delta<1$, we have with confidence $1-\frac\delta2$ that
$$
\cS_1(\bz,\lambda,g)\le C_1\left(\frac{\lambda^{\frac{2(s-1)}{1+s}}}{m}+\frac{\lambda^{\frac{2s-1}{1+s}}}{\sqrt{m}}\right)\log\frac2{\delta}
$$
for some positive constant $C_1$. If $s\ge1$ then $\cS_1(\bz,\lambda,f_\rho)=0$.
\end{lem}

For $\cS_2(\bz,\lambda)$, we cite the following result from \cite{Xiao2010}.

\begin{lem}\label{samplingerror2}
Suppose that (\ref{lip}) and (\ref{covering}) hold true. If $\lambda\le1$ then we have with confidence $1-\frac\delta2$ that
$$
\cS_2(\bz,\lambda)\le \frac12(\cE(f_{\bz,\lambda})-\cE(f_\rho))+C_2\frac{\log\frac2\delta+\log(1+m)}{\lambda^2}m^{-\frac1{1+\eta/\alpha}}
$$
for some positive constant $C_2$.
\end{lem}

Combining Lemmas \ref{twoterms}, \ref{regularizationerror}, \ref{samplingerror1}, and \ref{samplingerror2}, we reach a new learning rate estimate of the $\ell^1$-regularized least square regression.

\begin{thm}\label{newrate}
Suppose that $X$ satisfy (\ref{covering}), the output is bounded by a positive constant almost surely, and $f_\rho\in \ran(L_K^s)$ for some $s>0$. Let $K$ be a positive-definite reproducing kernel satisfying $\overline{\span}\{K(\cdot,x):x\in X\}=C(X)$, the condition (\ref{H2condition}) and the Lipschitz condition (\ref{lip}). Then there exists some constant $C>0$ such that with the choice $\lambda=m^{-\frac12\frac1{1+\eta/\alpha}\frac{1+s}{1+2s}}$, we have for all $0<\delta<1$ with confidence $1-\delta$ that
\begin{equation}\label{newrate1s}
\cE(f_{\bz,\lambda})-\cE(f_\rho)\le Cm^{-\frac{s}{1+2s}\frac1{1+\eta/\alpha}}\log\frac{2+2m}{\delta},\quad\mbox{ when }0<s<1
\end{equation}
and
$$
\cE(f_{\bz,\lambda})-\cE(f_\rho)\le Cm^{-\frac1{3(1+\eta/\alpha)}}\log\frac{1+m}{\delta},\quad\mbox{ when }s\ge1.
$$
\end{thm}
\begin{proof}
We only discuss the case when $0<s<1$ as the other situation is easier and can be shown in a similar way. We choose $\lambda=m^{-\theta}$, $\theta>0$ and get by Lemmas \ref{twoterms}, \ref{regularizationerror}, \ref{samplingerror1}, and \ref{samplingerror2} that there exists some constant $C>0$ such that with confidence $1-\delta$
\begin{equation}\label{newrateeq1}
\cE(f_{\bz,\lambda})-\cE(f_\rho)\le Cm^{-\gamma}\log\frac{2+2m}{\delta},
\end{equation}
where
$$
\gamma=\min\biggl\{\frac1{1+\eta/\alpha}-2\theta,1-\frac{2\theta(1-s)}{1+s},\frac12-\frac{1-2s}{1+s}\theta,\frac{2\theta s}{1+s}\biggr\}.
$$
The maximum of $\gamma$ is achieved when
$$
\theta=\frac12\frac1{1+\eta/\alpha}\frac{1+s}{1+2s}.
$$
Substituting the above choice into (\ref{newrateeq1}) yields (\ref{newrate1s}).
\end{proof}

Improvements of the learning rate can be achieved if higher regularity is imposed on the kernel $K$ \cite{Zhou2003} or better estimates of the sampling error are engaged \cite{Shi2011}. Another remark is that the assumption of positive-definiteness and symmetry on $K$ might be abandoned by using the strategy in \cite{Xiao2010}.

{\small
\bibliographystyle{abbrv}
\bibliography{BanachReg}
}

\end{document}